\RequirePackage{fix-cm}
\documentclass[smallextended]{svjour3}       
\smartqed  
\usepackage{graphicx}

\usepackage{hyperref}
\usepackage{cite}
\usepackage{amsmath,amssymb,amsfonts}
\usepackage{graphicx}
\usepackage{textcomp}
\usepackage{multirow,bigstrut}
\usepackage{algorithm}
\usepackage{algorithmicx}
\usepackage{algpseudocode}
\usepackage{amsmath}

\usepackage[misc]{ifsym}

\begin{document}

\title{A Robust Color Edge Detection Algorithm Based on Quaternion Hardy Filter}


\author{Wenshan Bi  \and Dong Cheng   \and {Wankai Liu} \and {Kit Ian Kou}
}


\institute{Wenshan Bi \at
              Department of Mathematics, Faculty of Science and Technology, University of Macau, Macau, China \\
              \email{wenshan0608@163.com}           
           \and
           Dong Cheng \at
              Research Center for Mathematics and Mathematics Education, Beijing Normal University, Zhuhai, 519087, China
              \email{chengdong720@163.com}
               \and
           Wankai Liu \at
              School of Mathematics and Quantitative Economics, Shandong University of Finance and Economics, Jinan, 250014, China
              \email{zjnulwk@163.com}
               \and     
           Corresponding author: Kit Ian Kou \at
              Department of Mathematics, Faculty of Science and Technology, University of Macau, Macau, China\\
              \email{kikou@um.edu.mo}
}

\date{Received: date / Accepted: date}

\maketitle

\begin{abstract}
This paper presents a robust filter called quaternion Hardy filter (QHF) for color image edge detection. The  QHF  can be capable of color edge feature enhancement and noise resistance. It is flexible to use  QHF by selecting suitable parameters to handle  different levels of noise.   In particular, the quaternion analytic signal, which is an effective tool in color image processing, can also be produced by quaternion Hardy filtering with specific parameters. Based on the QHF and the improved Di Zenzo gradient operator, a novel color edge detection algorithm is proposed. Importantly, it can be efficiently implemented by using the fast discrete quaternion Fourier transform technique.
{From the experimental results,
we conclude that the minimum PSNR improvement rate is $2.3\%$ and minimum SSIM improvement rate  is $30.2\%$ on the Dataset 3.}
The experiments demonstrate that the proposed algorithm outperforms several widely used algorithms.
\end{abstract}

\section{Introduction}
\label{sec:intro}
{Edge} detection is a fundamental problem in computer vision \cite{bibre2, bibre2a, bibre3}.
{It has a wide range of applications, including  medical imaging \cite{medical}, lane detection \cite{lane1},  face recognition \cite{face}, weed detection \cite{weed} and deep learning \cite{bibre1}, the well known method, plays an essential role in image processing and data analysis \cite{ZT1}-\cite{d2}.}
\subsection{Related works}
{Canny, differential phase congruence (DPC), and modified differential phase congruence (MDPC) detectors have drawn wide attention and achieved great success in gray-scale edge detection \cite{Sobel}-\cite{MDPC}.}
Another  approach of edge detection is detecting edges independently in each of the three color channels, and then obtain the final edge map by combining three single channel edge results  according to some proposed rules \cite{separate}.
{For example, fast Fourier transformation can also process color images, but actually requires more multiplications and additions than the quaternion Fourier transform \cite{bib19}.}
{However, these methods ignore the relationship between different color channels of the image. Instead of separately computing the scaled gradient for each color component, a multi-channel gradient edge detector has been widely used since it was proposed by Di Zenzo \cite{B19}. Later, Jin \cite{IDZ} solved the uncertainty of the Di Zenzo gradient direction and presented an improved Di Zenzo (IDZ) gradient operator, which achieves a significant improvement over DZ. However, the IDZ algorithm is not suitable for the edge detection of noisy images.}

\begin{figure}[th]
 \centering
 \includegraphics[height=4.15cm,width=8.3cm]{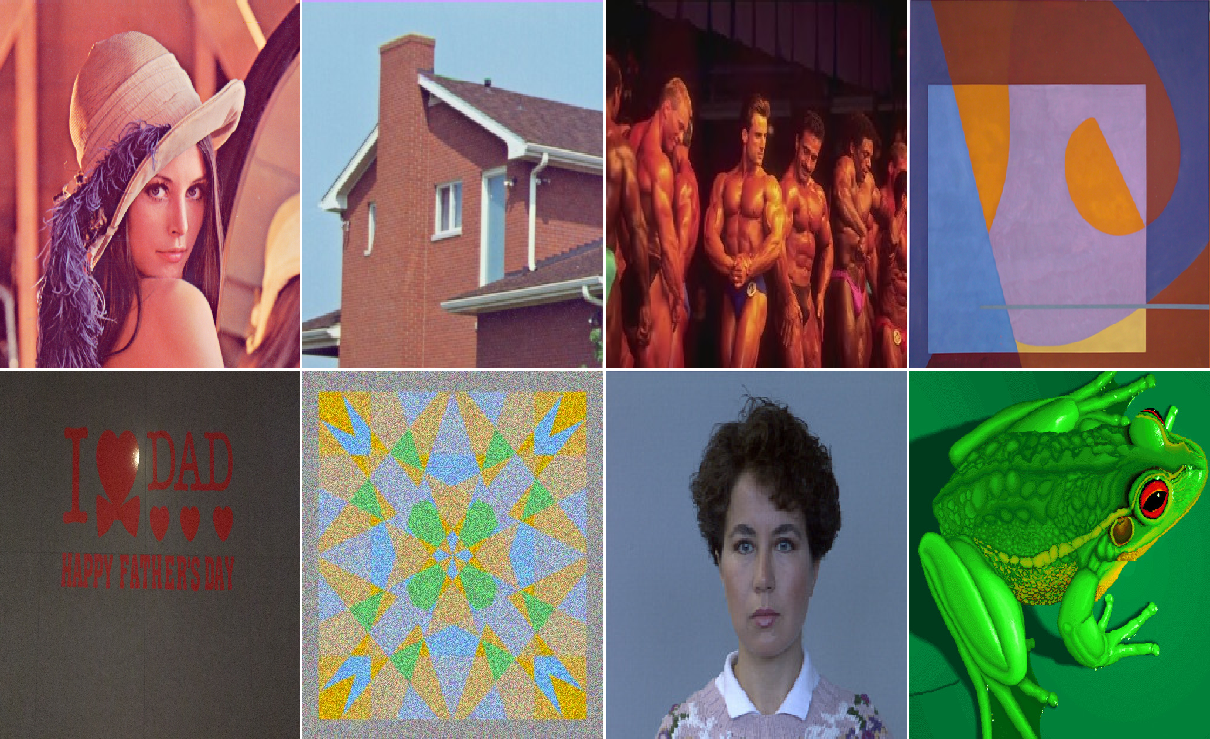}
  \caption{{The test images (namely Classic). From left to right and top to bottom: Lena, House, Men, T1, T2, T3, Cara and Frog.\label{TU1} \textbf{The figure is viewed better in zoomed PDF.}}}
\end{figure}

A growing number of research \cite{B21}-\cite{B24} indicates that quaternions are well adapted to color image processing by encoding color channels into three imaginary parts. The quaternion analytic signals are the boundary values of the functions in quaternion Hardy space \cite{B25}. Based on the quaternion analytic signal, researchers in \cite{B26} proposed some phase-based algorithms to detect the edge map of gray-scale images. It is shown that the introducing of quaternion analytic signal can reduce the influence of noise on edge detection results. It should be noted that although the tool of quaternion was applied, the algorithms (QDPC and QDPA) in \cite{B26} only considered the gray-scale images. Based on the quaternion Hardy filter and the improved Di Zenzo gradient operator, we propose a novel  edge detection algorithm,  which can be applied to color image.

\subsection{Paper contributions}
The contributions of this paper are summarized as follows.
\begin{enumerate}
  \item We propose a novel filter, named quaternion Hardy filter (QHF), for color image processing. Compared with quaternion analytic signal, our method has a better performance due to the flexible parameter selection of QHF.
  \item Based on the QHF and the improved Di Zenzo gradient operator, we propose a robust color edge detection algorithm. It can enhance the color edge in a holistic manner by extracting the main features of the color image.
  \item We set up a series of experiments to verify the denoising performance of the proposed algorithm in various environments. Visual and quantitative analysis are both conducted. Three widely used edge detection algorithms, Canny, Sobel and Prewitt, and two recent edge detection algorithms, QDPC, QDPA, DPC, and MDPC, are compared with the proposed algorithm.
      In terms of peak SNR (PSNR) and similarity index measure (SSIM), the proposed algorithm presents  the superiority in  edge detection.
\end{enumerate}
\subsection{Paper outlines}
The rest of the paper is organized as follows. Section \ref{sec:Pre} recalls some preliminaries of the improved Di Zenzo gradient operator, quaternions, quaternion Fourier transform, quaternion Hardy space and quaternion analytic signal. Section \ref{sec:Pro} presents the main result of the paper, it defines the novel algorithm for color-based edge detection of real-world images. Experimental results of the proposed algorithm are shown in Section \ref{sec:Exp}. Conclusions and discussions of the future work are drawn in Section \ref{sec:Con}.
\section{Preliminaries}
\label{sec:Pre}
\begin{figure}[t]
 \centering
 \includegraphics[height=2.075cm,width=6.225cm]{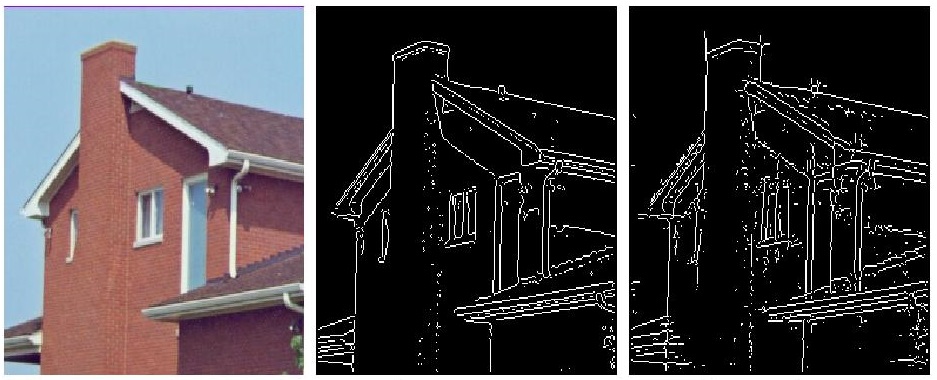}
  \caption{The noiseless House image (left). The edge maps obtained by IDZ gradient algorithm (middle) and the proposed algorithm (right). \textbf{The figure is viewed better in zoomed PDF.}}
  \label{house1}
\end{figure}
This part recalls some preparatory knowledge of the improved Di Zenzo gradient operator \cite{IDZ}, quaternions, quaternion Fourier transform \cite{B27}, quaternion Hardy space \cite{B26}, and quaternion analytic signal \cite{B28} which will be used throughout the paper.
\subsection{The improved Di Zenzo gradient operator}
\label{Pre-IDZ}
{In this section, we recall the improved Di Zenzo gradient operator, namely the IDZ gradient operator. It will be combined with the quaternion Hardy filter  to establish the novel edge detection algorithm in the next section.}

Let $f$  be an $M\times{N}$ color image that maps a point $(x_1,x_2)$  to a vector $(f_1(x_1,x_2),f_2(x_1,x_2),f_3(x_1,x_2))$.
Then the square of the variation of $f$ at the position $(x_1,x_2)$ with the distance $\gamma$ in the direction $\theta$ is given by
\begin{eqnarray}\label{df2}
\begin{aligned}
df^2& =\|f(x_1+\gamma{\cos\theta,x_2+\gamma{\sin\theta})-f(x_1,x_2)}\|_2^2 \\
&\approx\sum\limits_{i=1}^{3}\left(\frac{\partial{f_i}}{\partial{x_1}}\gamma\cos\theta+
\frac{\partial{f_i}}{\partial{x_2}}\gamma\sin\theta\right)^2\\
&=\gamma^2f(\theta),
\end{aligned}
\end{eqnarray}
{where $\|\cdot\|$ represents 2-norm, $\|\mathbf{X}\|_2=\sqrt{{x_1}^2+\ldots+{x_n}^2}$, $\mathbf{X}=(x_1,\ldots,x_n)$, and}
\begin{eqnarray}\label{ftheta}
 \begin{aligned}
f(\theta) =&2\sum\limits_{i=1}^{3}\frac{\partial{f_i}}{\partial{x_1}}\frac{\partial{f_i}}{\partial{x_2}}\cos\theta\sin\theta\\
&+\sum\limits_{i=1}^{3}\left(\frac{\partial{f_i}}{\partial{x_1}}\right)^2\cos^2\theta+\sum\limits_{i=1}^{3}\left(\frac{\partial{f_i}}{\partial{x_2}}\right)^2\sin^2\theta.
\end{aligned}
\end{eqnarray}
\begin{figure}[t]
\centering
 \includegraphics[height=6.225cm,width=8.3cm]{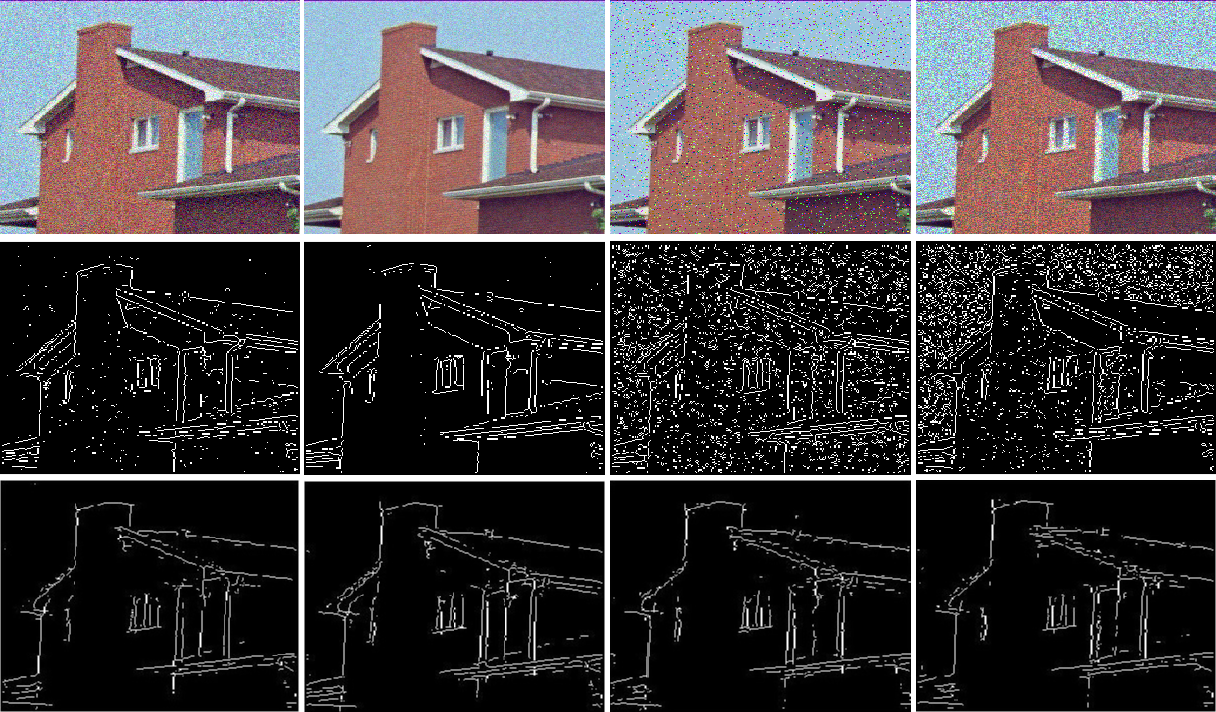}
  \caption{ The first row is the noisy House image with additive Gaussian noise, Poisson noise, Salt \& Pepper noise and Speckle noise from left to right. The second and third rows are the edge maps which are captured by IDZ  algorithm and the proposed algorithm, respectively. \textbf{The figure is viewed better in zoomed PDF.}}
  \label{house2}
\end{figure}
Let

\begin{figure}[th]
 \centering
 \includegraphics[height=8.3cm,width=8.3cm]{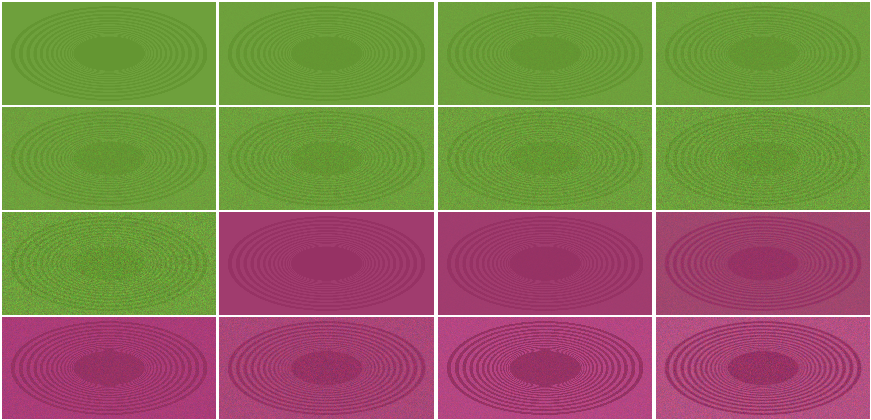}
  \caption{{{The test images. These images are randomly selected from the public image dataset (namely Synthetic, and from \href{https://urldefense.proofpoint.com/v2/url?u=http-3A__color.univ-2Dlille.fr_datasets_color-2Dedge&d=DwIFaQ&c=KXXihdR8fRNGFkKiMQzstu-8MbOxd1NuZkcSBymGmgo&r=WuCkcl8HOAexOSwshlwJ3w&m=GpWzZCbpvddgVNECF4m_mlofU3pS9zmK0WK-eRKLN2Q&s=_GRVRNnmdJv7DBtB1vHs6n-mw2OpOT-FiCx5e3fl0p4&e=}
  {link 2})}
  . \textbf{The figure is viewed better in zoomed PDF.}}}
  \label{TU2}
\end{figure}
\begin{eqnarray}\label{ABC}
 \left\{
 \begin{aligned}
&A= \sum\limits_{i=1}^{3}\left(\frac{\partial{f_i}}{\partial{x_1}}\right)^2; \\
&B= \sum\limits_{i=1}^{3}\left(\frac{\partial{f_i}}{\partial{x_2}}\right)^2; \\
&C= \sum\limits_{i=1}^{3}\frac{\partial{f_i}}{\partial{x_1}}\frac{\partial{f_i}}{\partial{x_2}}.
\end{aligned}\right.
\end{eqnarray}

Then the gradient magnitude $f_{\max}$ of the improved Di Zenzo's gradient operator is given by
\begin{eqnarray}\label{fMAX}
\begin{aligned}
&f_{\max}(\theta_{\max})\\
=&\max_{0\leq \theta \leq 2 \pi}{f(\theta)}\\
=&\frac{1}{2}\bigg(A+C+\sqrt{(A-C)^2+(2B)^2}\bigg).
\end{aligned}
\end{eqnarray}

The gradient direction is defined as the value $\theta_{\max}$ that maximizes $f(\theta)$ over $0\leq \theta \leq 2 \pi$
\begin{eqnarray}\label{thetaMAX}
\theta_{\max}=&\mbox{sgn}(B)\arcsin\bigg(\frac{f_{\max}-A}{2f_{\max}-A-C}\bigg),
\end{eqnarray}
where
$(A-C)^2+B^2\neq0$,
$\mbox{sgn}(B)=\left\{
          \begin{array}{ll}
            1, & {B\geq0;} \\
            -1, & {B<0.}
          \end{array}
        \right.$
When $(A-C)^2+B^2=0$, $\nonumber\theta_{\max}$ is undefined.

It is important to note that the IDZ edge detector is designed to process real domain signals and don't possess the capability of de-noising.

\subsection{Quaternions}
\label{Pre-Q}
As a natural extension of the complex space $\mathbb{C}$, the quaternion space $\mathbb{H}$ was first proposed by Hamilton \cite{Hamilton}. A complex number consists of two components: one real part and one imaginary part. While a quaternion $q\in\mathbb{H}$ has four components, i.e., one real part and three imaginary parts

\begin{figure}[th]
 \centering
 \includegraphics[height=4.15cm,width=8.3cm]{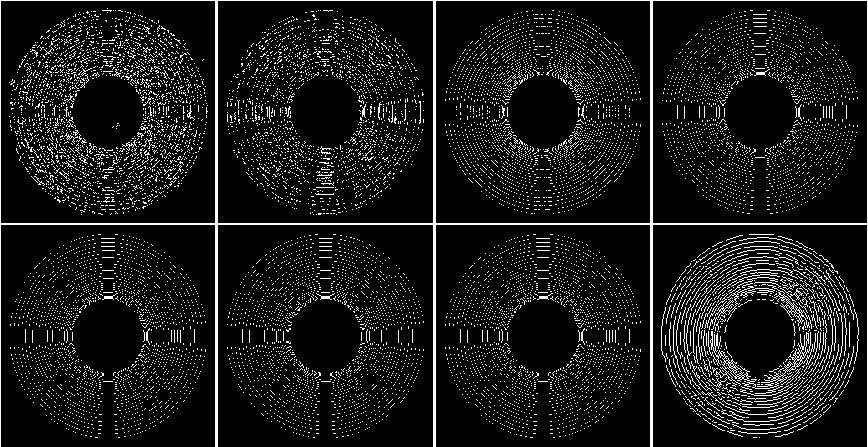}
  \caption{{Edge maps of a randomly selected noiseless image (the first row and second column of Fig. \ref{TU2}, namely GREEN). From left to right and top to bottom: QDPC, QDPA, Canny, Sobel, Prewitt, DPC, MDPC and the proposed algorithm. \textbf{The figure is viewed better in zoomed PDF.}}}
  \label{NL0}
\end{figure}


\begin{equation}\label{Eq.H}
q=q_0+q_1\mathbf{i}+q_2\mathbf{j}+q_3\mathbf{k},
\end{equation}
where $q_{n}\in\mathbb{R}, n=0,1,2,3$, and the basis elements $\{\mathbf{i},\mathbf{j},\mathbf{k}\}$ obey the Hamilton's multiplication rules
\begin{eqnarray}\label{Eq.Hamilton's multiplication rules}
\begin{aligned}
\mathbf{i}^2&=\mathbf{j}^2=\mathbf{k}^2=\mathbf{i}\mathbf{j}\mathbf{k}=-1;\\
\mathbf{i}\mathbf{j}&=\mathbf{k},\mathbf{j}\mathbf{k}=\mathbf{i},\mathbf{k}\mathbf{i}=\mathbf{j};\\
\mathbf{j}\mathbf{i}&=-\mathbf{k},\mathbf{k}\mathbf{j}=-\mathbf{i},\mathbf{i}\mathbf{k}=-\mathbf{j}.\\
\end{aligned}
\end{eqnarray}

Given a quaternion $q=q_0+q_1\mathbf{i}+q_2\mathbf{j}+q_3\mathbf{k}$, its quaternion conjugate is $\overline{q}:=q_0-q_1\mathbf{i}-q_2\mathbf{j}-q_3\mathbf{k}$. We write $\mathbf{Sc}(q):=\frac{1}{2}(q+\overline{q})=q_{0}$ and
$\mathbf{Vec}(q):=\frac{1}{2}(q-\overline{q})=q_1\mathbf{i}+q_2\mathbf{j}+q_3\mathbf{k}$,
which are the scalar and vector parts of $q$ , respectively. This leads to a modulus of $q\in\mathbb{H}$ defined by
\begin{equation}\label{Eq.modulus}
|q|=\sqrt{q\overline{q}}=\sqrt{\overline{q}q}=\sqrt{q_0^{2}+q_1^{2}+q_2^{2}+q_3^{2}},
\end{equation}
where $q_{n}\in\mathbb{R}, n=0,1,2,3$.

By \eqref{Eq.H}, an $\mathbb{H}$-valued function $f:\mathbb{R}^{2}\rightarrow\mathbb{H}$ can be expressed as
\begin{eqnarray}\label{Eq.H-valued function}
\begin{aligned}
&f(x_{1},x_{2})\\
=&f_0(x_{1},x_{2})+f_1(x_{1},x_{2})\mathbf{i}+f_2(x_{1},x_{2})\mathbf{j}+f_3(x_{1},x_{2})\mathbf{k},
\end{aligned}
\end{eqnarray}
where $f_{n}:\mathbb{R}^{2}\rightarrow\mathbb{R}(n=0,1,2,3)$.
In this paper, we consider using $f(x_{1},x_{2})$ to represent a color image, i.e $f(x_{1},x_{2})=f_1(x_{1},x_{2})\mathbf{i}+f_2(x_{1},x_{2})\mathbf{j}
+f_3(x_{1},x_{2})\mathbf{k}$. While $f_1$, $f_2$ and $f_3$ represent respectively the red, green and blue components of the color image.
\subsection{Quaternion Fourier transform}
\label{Pre-FT}

Suppose that $f$ is an absolutely integrable complex function defined on $\mathbb{R}$, then the Fourier transform \cite{stein} of $f$ is given by
\begin{eqnarray}\label{Eq.Fourier transform}
\widehat{f}(w)=\frac{1}{\sqrt{2\pi}}\int_{\mathbb{R}}f(x)e^{-iwx}d{x},
\end{eqnarray}
where $w$ denotes the angular frequency. Moreover, if  $\widehat{f}$ is an absolutely integrable complex function defined on $\mathbb{R}$ , then $f$ can be reconstructed by the Fourier transform of $f$ and is expressed by
\begin{eqnarray}\label{Eq.inverse Fourier transform}
f(x) =\frac{1}{\sqrt{2\pi}}\int_{\mathbb{R}}\widehat{f}(w)e^{iwx}dw.
\end{eqnarray}
\begin{figure}[ht]
 \centering
 \includegraphics[height=4.15cm,width=8.3cm]{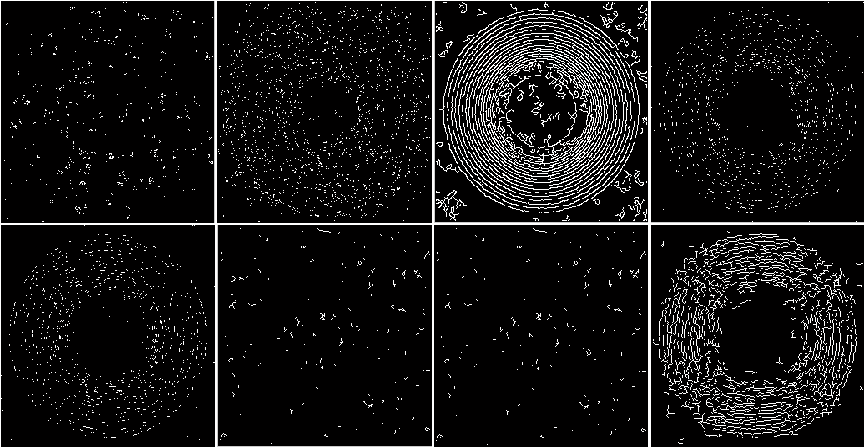}
  \caption{{The edge maps of the noisy Green image (with additive Gaussian noise) given by different algorithms. From left to right and top to bottom: QDPC, QDPA, Canny, Sobel, Prewitt, DPC, MDPC and the proposed algorithm. \textbf{The figure is viewed better in zoomed PDF.}}}
  \label{NL1}
\end{figure}
\begin{figure}[ht]
 \centering
 \includegraphics[height=4.15cm,width=8.3cm]{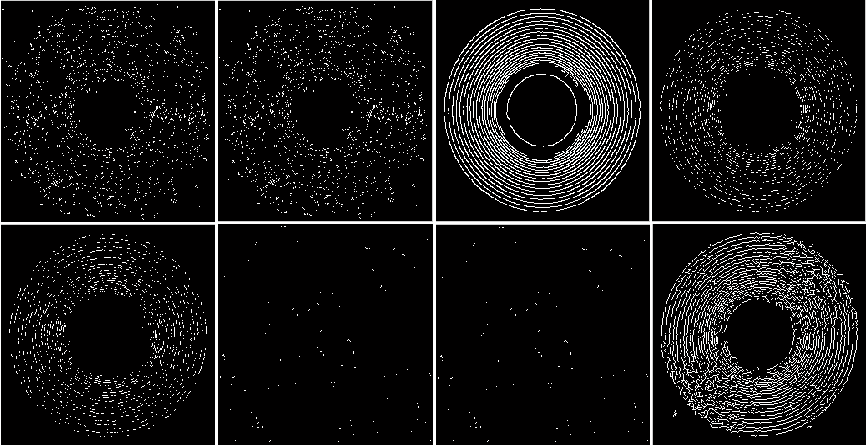}
  \caption{{The edge maps of the noisy Green image (with additive Poisson noise) given by different algorithms. From left to right and top to bottom: QDPC, QDPA, Canny, Sobel, Prewitt, DPC, MDPC and the proposed algorithm. \textbf{The figure is viewed better in zoomed PDF.}}}
  \label{NL2}
\end{figure}
\begin{figure}[ht]
 \centering
 \includegraphics[height=4.15cm,width=8.3cm]{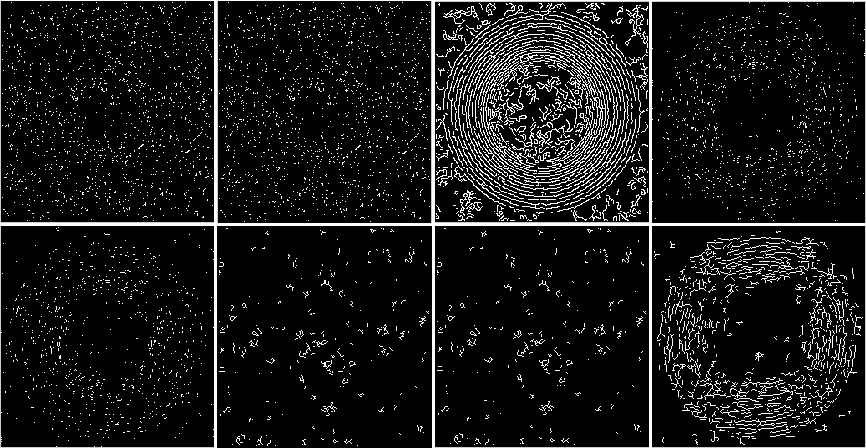}
  \caption{{The edge maps of the noisy Green image (with additive Salt \& Pepper noise) given by different algorithms. From left to right and top to bottom: QDPC, QDPA, Canny, Sobel, Prewitt, DPC, MDPC and the proposed algorithm. \textbf{The figure is viewed better in zoomed PDF.}}}
  \label{NL3}
\end{figure}

The quaternion Fourier transform, regarded as an extension of Fourier transform in quaternion domain, plays a vital role in grayscale image processing. The first definition of the quaternion Fourier transform was given in \cite{ELL} and the first application to color images was discussed in \cite{B32}. It was recently applied to find the envelope of the image \cite{B33}. The application of quaternion Fourier transform on color images was discussed in \cite{B24,B34}. The Plancherel and inversion theorems of quaternion Fourier transform in the square integrable signals class was established in \cite{B35}.
{Due to the non-commutativity of the quaternions, there are various types of quaternion Fourier transforms. In the following, we focus our attention on the two-sided quaternion Fourier transform (QFT).}

\begin{figure}[ht]
 \centering
 \includegraphics[height=4.15cm,width=8.3cm]{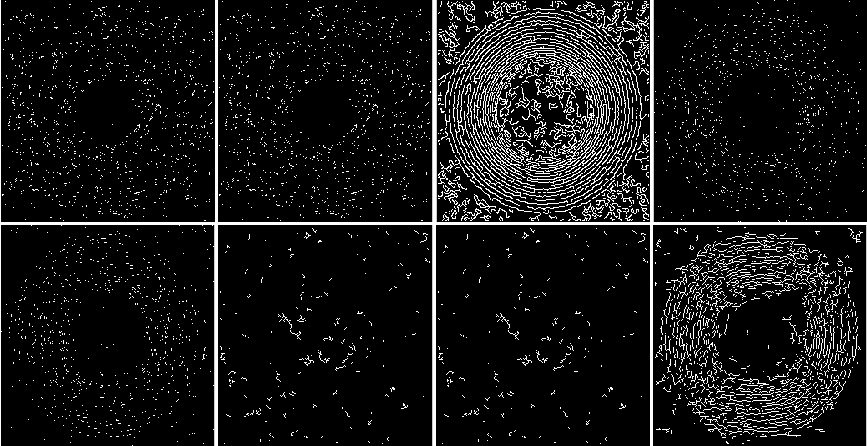}
  \caption{{The edge maps of the noisy Green image (with additive Speckle noise) given by different algorithms. From left to right and top to bottom: QDPC, QDPA, Canny, Sobel, Prewitt, DPC, MDPC and the proposed algorithm. \textbf{The figure is viewed better in zoomed PDF.}}}
  \label{NL4}
\end{figure}
\begin{figure}[ht]
 \centering
 \includegraphics[height=8.3cm,width=8.3cm]{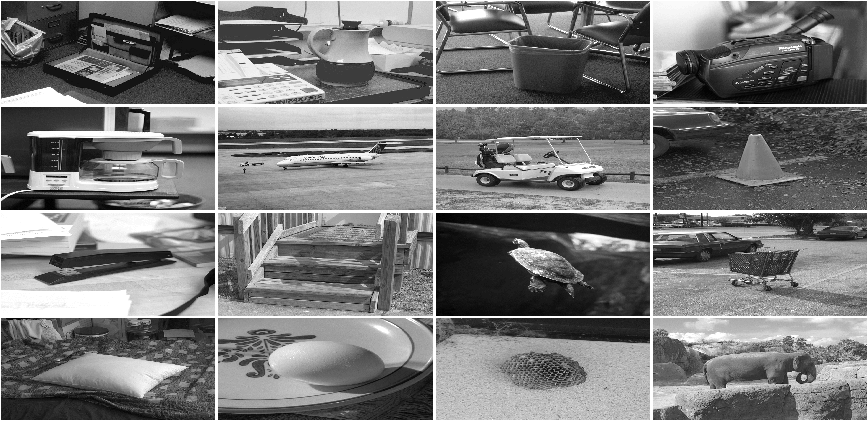}
  \caption{{The test images. These images are randomly selected from the public image dataset (namely CSEE, which is from \cite{bibre2} and \cite{bibre2a}). \textbf{The figure is viewed better in zoomed PDF.}}
  .}
  \label{TU3}
\end{figure}

\begin{figure}[ht]
 \centering
 \includegraphics[height=4.15cm,width=8.3cm]{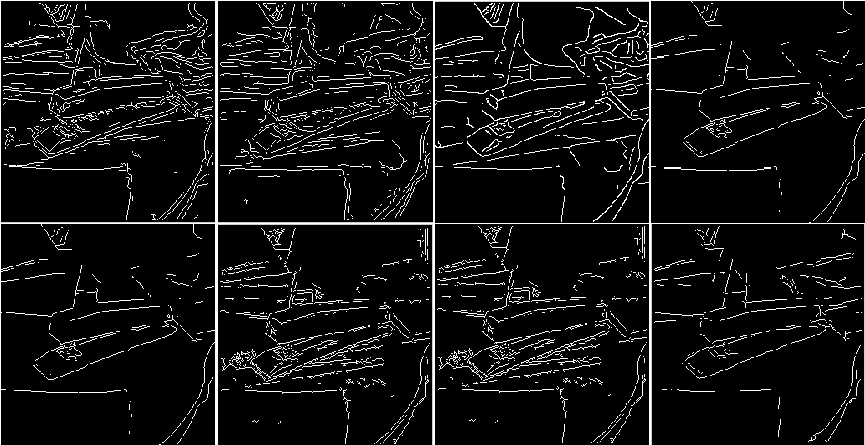}
  \caption{{Edge maps of a randomly selected noiseless image ( from CSEE dataset, namely Stapler) given by different algorithms. From left to right and top to bottom: QDPC, QDPA, Canny, Sobel, Prewitt, DPC, MDPC and the proposed algorithm. \textbf{The figure is viewed better in zoomed PDF.}}}
  \label{N30}
\end{figure}


\begin{figure}[ht]
 \centering
 \includegraphics[height=4.15cm,width=8.3cm]{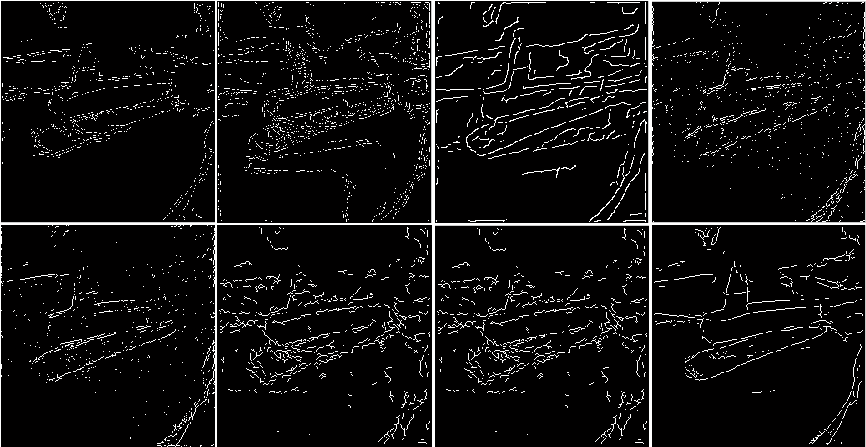}
  \caption{{The edge maps of the noisy Stapler image (with additive Gaussian noise) given by different algorithms. From left to right and top to bottom: QDPC, QDPA, Canny, Sobel, Prewitt, DPC, MDPC and the proposed algorithm. \textbf{The figure is viewed better in zoomed PDF.}}}
  \label{N31}
\end{figure}
Suppose that $f$ is an absolutely integrable $\mathbb{H}$-valued function defined on $\mathbb{R}^2$, then the continuous quaternion Fourier transform of $f$ is defined by
\begin{equation}\label{Eq.2-sided QFT}
\begin{aligned}
&(\mathcal{F}f)(w_1,w_2)\\
=&\frac{1}{2\pi}\int_{\mathbb{R}^{2}}e^{-\mathbf{i}w_1x_1}f(x_1,x_2)e^{-\mathbf{j}w_2x_2}d{x_1}d{x_2},
\end{aligned}
\end{equation}
where $w_l$ and $x_l$ denote the 2D angular frequency and 2D space ($l=1,2$), respectively.

Furthermore, if $f$ is an absolutely integrable $\mathbb{H}$-valued function defined on $\mathbb{R}^2$, then the continuous inverse quaternion Fourier transform (IQFT) of $f$ is defined by

\begin{equation}\label{Eq.inverse 2-sided QFT}
\begin{aligned}
&({\mathcal{F}}^{-1}f)(x_1,x_2)\\
=&\frac{1}{2\pi}\int_{\mathbb{R}^{2}}e^{\mathbf{i}w_1x_1}f(w_1,w_2)e^{\mathbf{j}w_2x_2}d{w_1}d{w_2},
\end{aligned}
\end{equation}
where $w_l$ and $x_l$ denote the 2D angular frequency and 2D space ($l=1,2$), respectively.

The discrete quaternion Fourier transform (DQFT) and its inverse is introduced by Sangwine \cite{DQFT}. Suppose that the discrete array $f(m,n)$ is of dimension $M\times N$. The DQFT has the following form
\begin{equation}\label{DQFT}
\begin{aligned}
&\mathcal{F}_D[f](p,s)\\
=&\frac{1}{\sqrt{MN}}
{\sum_{m=0}^{M-1}}{\sum_{n=0}^{N-1}}
e^{-\mathbf{i}2\pi \frac{mp}{M}}f(m,n)e^{-\mathbf{j}2\pi \frac{ns}{N}}.
\end{aligned}
\end{equation}
{where $p$ and $s$ are the variables in frequency domain.} And the inverse discrete quaternion Fourier transform (IDQFT) is
\begin{equation}\label{IDQFT}
\begin{aligned}
&f(m,n)\\
=&\frac{1}{\sqrt{MN}}
{\sum_{p=0}^{M-1}}{\sum_{s=0}^{N-1}}
e^{\mathbf{i}2\pi \frac{mp}{M}}\mathcal{F}_D[f](p,s)e^{\mathbf{j}2\pi \frac{ns}{N}}.
\end{aligned}
\end{equation}
\subsection{Quaternion Hardy space}
\label{Pre-QHS}

Let $ {\mathbb{C}}= \{z|z=x+si, x, s \in \mathbb{R}\}$ be the complex plane and a subset of ${\mathbb{C}}$ is defined by ${\mathbb{C}}^{+}= \{z|z=x+si,
x, s \in \mathbb{R},s>0\}$, namely upper half complex plane.
The Hardy space ${\mathbf{H}}^2(\mathbb{C}^{+})$ on the upper half complex plane consists of functions $c$ satisfying the following conditions
\begin{eqnarray}\label{Eq.Hardy space}
 \left\{
 \begin{aligned}
&\frac{\partial}{\partial \overline{z}}c(z)=0;\\
&(\sup\limits_{s>0}\int_{\mathbb{R}} |c(x+si)|^{2}dx)^{\frac{1}{2}} < \infty.
\end{aligned}\right.
\end{eqnarray}

The generalization \cite{B24} to higher dimension is given as follows. Let $ \mathbb{C}_{\mathbf{i} \mathbf{j}}= \{( z_1, z_2)|z_1=x_1+s_1\mathbf{i}, z_2=x_2+s_2\mathbf{j},
x_l, s_l \in \mathbb{R}, l=1,2\}$ and a subset of $\mathbb{C}_{\mathbf{i} \mathbf{j}}$ is defined by $ \mathbb{C}_{\mathbf{i} \mathbf{j}}^{+}= \{( z_1, z_2)|z_1=x_1+ s_1\mathbf{i}, z_2=x_2+ s_2\mathbf{j},x_l, s_l \in \mathbb{R}, s_l> 0, l=1,2\}.$
The quaternion Hardy space $\mathbf{Q}^2(\mathbb{C}_{\mathbf{i} \mathbf{j}}^+)$ consists of all functions satisfying the following conditions
\begin{eqnarray}\label{QHS conditions}
 \left\{
 \begin{aligned}
 &\frac{\partial}{\partial \overline{z_{1}}}h(z_1, z_2)=0;\\
 &h(z_1, z_2)\frac{\partial}{\partial \overline {z_{2}}}= 0;\\
 &(\sup\limits_{\substack{s_1>0 \\
 s_2>0}}\int_{\mathbb{R}^2} |h(x_1+ s_1\mathbf{i}, x_2+s_2\mathbf{j} )|^{2}dx_{1}dx_{2})^{\frac{1}{2}}  < \infty,
  \end{aligned}\right.
\end{eqnarray}
where $\frac{\partial}{\partial \overline{z_{1}}}=\frac{\partial}{\partial {x_{1}}}+\mathbf{i}\frac{\partial}{\partial {s_{1}}}$, $\frac{\partial}{\partial \overline{z_{2}}}=\frac{\partial}{\partial {x_{2}}}+\mathbf{j}\frac{\partial}{\partial {s_{2}}}$.

\subsection{Quaternion analytic signal}
\label{Pre-QAS}
In the following, we review the concept of analytic signal. Given a real signal $f$, combined with its own Hilbert transform, then the analytic signal of $f$ is given by
\begin{equation}\label{Eq.1D analytic signal}
f_a(x) =f(x)+i\mathcal{H}[f](x),   x\in\mathbb{R},
\end{equation}

where $\mathcal{H}[f]$ denotes the Hilbert transform of $f$ and is defined by
\begin{equation}\label{Eq.Hilbert transform}
\mathcal{H}[f](x) =\frac{1}{\pi} \lim\limits_{\varepsilon\to 0^+} \int_{\varepsilon \leq |x-s|} \frac{f(s)}{x-s}ds.
\end{equation}
\begin{figure}[ht]
 \centering
 \includegraphics[height=4.15cm,width=8.3cm]{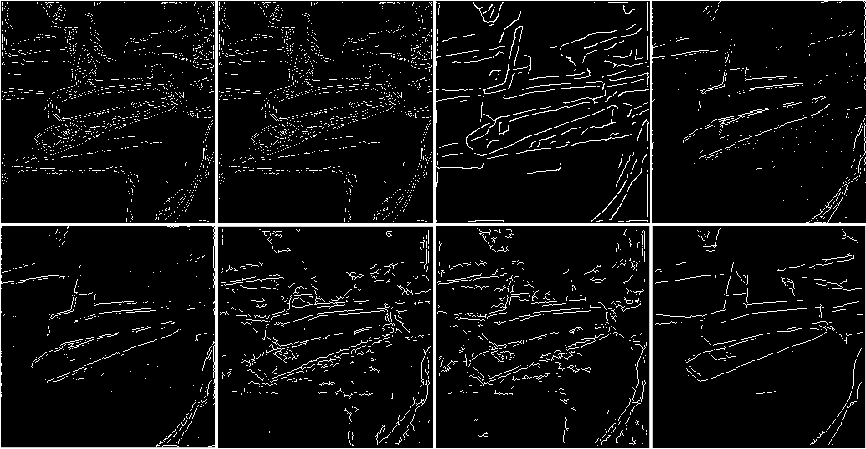}
  \caption{{The edge maps of the noisy Stapler image (with additive Poisson noise) given by different algorithms. From top to bottom: QDPC, QDPA, Canny, Sobel, Prewitt, DPC, MDPC and the proposed algorithm. \textbf{The figure is viewed better in zoomed PDF.}}}
  \label{N32}
\end{figure}

\begin{figure}[ht]
 \centering
 \includegraphics[height=4.15cm,width=8.3cm]{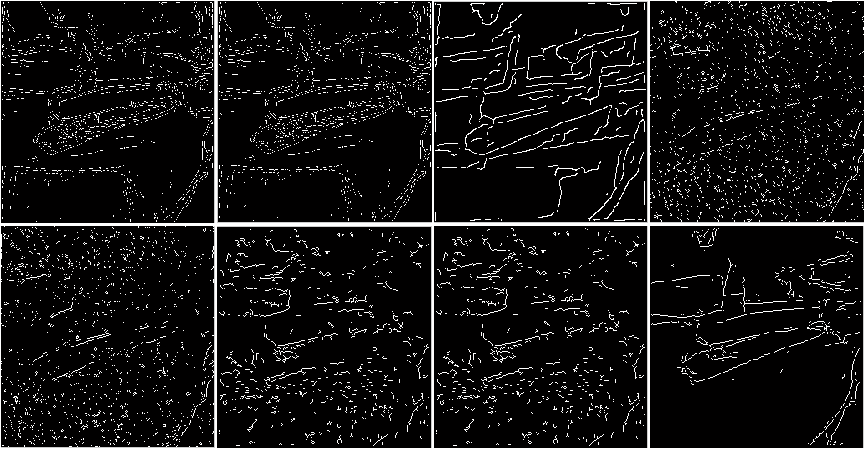}
  \caption{{The edge maps of the noisy Stapler image (with additive  Salt \& Pepper noise) given by different algorithms. From top to bottom: QDPC, QDPA, Canny, Sobel, Prewitt, DPC, MDPC and the proposed algorithm. \textbf{The figure is viewed better in zoomed PDF.}}}
  \label{N33}
\end{figure}

\begin{figure}[ht]
 \centering
 \includegraphics[height=4.15cm,width=8.3cm]{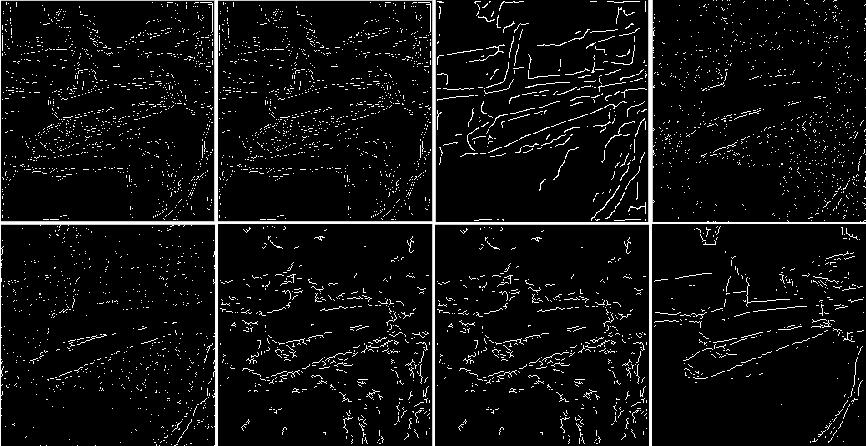}
  \caption{{The edge maps of the noisy Stapler image (with additive Speckle noise) given by different algorithms. From top to bottom: QDPC, QDPA, Canny, Sobel, Prewitt, DPC, MDPC and the proposed algorithm. \textbf{The figure is viewed better in zoomed PDF.}}}
  \label{N34}
\end{figure}

\begin{figure}[ht]
 \centering
 \includegraphics[height=6.225cm,width=8.3cm]{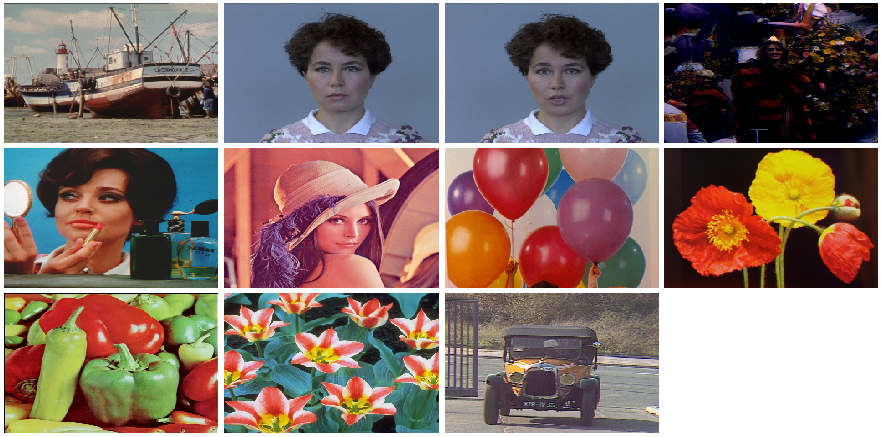}
  \caption{{The test images. These images are randomly selected from the public image dataset (namely Scene, and from \href{http://decsai.ugr.es/cvg/dbimagenes/}
  {link 4}). \textbf{The figure is viewed better in zoomed PDF.}}}
  \label{TU4}
\end{figure}

The Fourier transform of an analytic $f_a$ defined in (\ref{Eq.1D analytic signal}) is given by
\begin{equation}\label{fa}
\widehat{f_a}(w)=\left(1+\mbox{sgn}(w)\right)\widehat{f}(w),
 \end{equation}
 where $w\in\mathbb{R}$.

A natural extension of the analytic signal from 1D to 4D space in the quaternion setting is called quaternion analytic signal. It was proposed by B\"ulow and Sommer \cite{B28} using partial and total Hilbert transform associated with QFT. Given a 2D quaternion valued signal $f$, combined with its own quaternion partial and total Hilbert transform, then we get a quaternion analytic signal $f_q$ \cite{B28} as follows

\begin{equation}\label{QanalyticS}
\begin{aligned}
&f_q(x_1,x_2)\\
=&f(x_1,x_2)+\mathbf{i}\mathcal{H}_{x_1}[f](x_1)+\mathcal{H}_{x_2}[f](x_2)\mathbf{j}\\
&+\mathbf{i}\mathcal{H}_{{x_1}{x_2}}[f](x_1,x_2)\mathbf{j},
\end{aligned}
\end{equation}
where
\begin{eqnarray}
\begin{aligned}
&\mathcal{H}_{x_1}[f](x_1)=
\frac{1}{\pi}
\mathop{\lim}\limits_{\alpha\to0}{\int}_{|x_1-t_1|>\alpha}\frac{f(t_1,x_2)}{x_1-t_1}dt_1, \label{deHx1}\\
&\mathcal{H}_{x_2}[f](x_2)=
\frac{1}{\pi} \mathop{\lim}\limits_{\alpha\to0}{\int}_{|x_2-t_1|>\alpha}\frac{f(x_1,t_1)}{x_2-t_1}dt_1, \label{deHx2}\\
\end{aligned}
\end{eqnarray}

are the quaternion partial Hilbert transform of $f$ along the $x_1$-axis, $x_2$-axis, respectively. While
\begin{equation}\label{deHx1,x2}
\begin{aligned}
&\mathcal{H}_{{x_1}{x_2}}[f](x_1,x_2)\\
=&
\frac{1}{\pi}
\lim_{
\substack{\alpha_1\to0 \\
\alpha_2\to0}}
\int_{
\substack{|x_1-t_1|>\alpha_1\\
|x_2-t_2|>\alpha_2}}
\frac{f(t_1,t_2)dt_1dt_2}{(x_1-t_1)(x_2-t_2)},
\end{aligned}
\end{equation}
is the quaternion total Hilbert transform along the $x_1$ and $x_2$ axes.
By  direct computation, the quaternion Fourier transform of quaternion analytic signal   is given by
\begin{equation}\label{fq}
\begin{aligned}
(\mathcal{F}{f_q})(w_1,w_2)=&[1+\mbox{sgn}(w_1)][1+\mbox{sgn}(w_2)]\\
&(\mathcal{F}{f})(w_1,w_2).
\end{aligned}
\end{equation}
\section{Proposed algorithm}
\label{sec:Pro}
In this section, we introduce our new color edge detection algorithm. To begin with, the definition of quaternion Hardy filter is presented.

\subsection{Quaternion Hardy filter}
\label{pro-QHF}
The quaternion analytic signal $f_q$ can be regarded as the output signal of a filter with input $f$. The system function of this filter is
\begin{equation}\label{filter1}
H_1(w_1,w_2)=[1+\mbox{sgn}(w_1)][1+\mbox{sgn}(w_2)].
\end{equation}
In this paper, we use a novel filter, named {\it quaternion Hardy filter (QHF)}, to construct a high-dimensional analytic signal.  The system function of QHF is defined by
\begin{equation}\label{filter2}
\begin{aligned}
&H(w_1,w_2,s_1,s_2)\\
=&[1+\mbox{sgn}(w_1)][1+\mbox{sgn}(w_2)]
e^{-\mid{w_1}\mid{s_1}}e^{-\mid{w_2}\mid{s_2}},
\end{aligned}
\end{equation}
where $s_1\geq0, s_2\geq0$ are parameters of the system function. The factors $(1+\mbox{sgn}(w_1))(1+\mbox{sgn}(w_2))$ and $e^{-\mid{w_1}\mid{s_1}}e^{-\mid{w_2}\mid{s_2}}$ play different roles in quaternion Hardy filter. The former performs Hilbert transform on the input signal, while the later plays a role of suppressing the high-frequency. On the one hand, the Hilbert transform operation can  selectively emphasize the edge feature of
an input object. On the other hand, the low-pass filtering can improve the ability of noise immunity for the QHF. It can be seen that as increase with $s_1, s_2$, the effect of inhibiting for the high frequency becomes more obvious. In particular, if $s_1=s_2=0$, then $e^{-\mid{w_1}\mid{s_1}}e^{-\mid{w_2}\mid{s_2}}=1$, it follows that
\begin{equation}\label{filter2=0}
H(w_1,w_2,0,0)=H_1(w_1,w_2),
\end{equation}
which means that there is no effect in high frequency inhibiting.

Parameters $s_1$ and $s_2$ play the role of low-pass filtering in vertical and horizontal directions, respectively. When the signal frequencies in the two directions are similar, then $s_1$ and $s_2$ can be set to the same value. If the signal frequencies in these two directions are different, then $s_1$ and $s_2$ should be different. For example, if the horizontal noise in the image is large, the value of $s_2$ should be set larger to enhance the anti-noise ability in that direction.
This means that the QHF is very general and flexible, and it can solve many problems that can't be solved well by quaternion analytic signal.
\begin{algorithm}
	\caption{{Quaternion Hardy filter edge detection algorithm}}
	\begin{algorithmic}[1] %
		\Require Color image $f$, fixed parameters $s_{1} > 0$ and $s_{2} > 0$.
		\Ensure Edge map.
		
		\State  {Associate $f$ with a $\mathbb{H}$-valued signal.}
		\State  {Compute the DQFT of the $f$ using the following equation (\ref{DQFT}). The result will be $\mathcal{F}_D[f]$.}
		\State  Multiplying $\mathcal{F}_D[f]$ by the system function (\ref{filter2}) of the QHF. Then we obtain $(\mathcal{F}_D[f_H])$.
		\State   Compute the inverse DQFT for  $\mathcal{F}_D[f_H]$ by applying equation (\ref{IDQFT}), we obtain $f_H$.
		\State  Extract the vector part of $f_H$, we obtain $\mathbf{Vec}(f_H)$.
		\State   Perform the IDZ gradient operator to $\mathbf{Vec}(f_H)$.
		\State Output edge map.
	\end{algorithmic}
\end{algorithm}

\begin{figure}[ht]
 \centering
 \includegraphics[height=9.5cm,width=4.5cm]{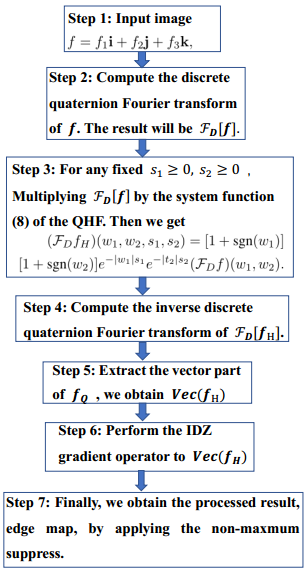}
  \caption{Block diagram of the proposed algorithm.}
  \label{flow}
\end{figure}
\begin{table*}[th]
\centering
\caption{The PSNR comparison values for Fig. \ref{TU1}.
 Types of noise: I- Gaussian noise, II- Poisson noise, III- Salt \& Pepper noise and IV- Speckle noise.}
\label{tabP-1}
\begin{tabular}{cccccccccc}
\hline
&  &QDPC \cite{B26} &QDPA \cite{B26}   &Canny \cite{Canny1986} &Sobel \cite{Sobel} &Prewitt \cite{Prewitt}  &DPC \cite{DPC}  &MDPC \cite{MDPC}   &Ours\\
\hline
&\uppercase\expandafter{\romannumeral1} &{65.0923} &{60.8565}	&{56.4596}	&\textbf{65.1202}	&{64.7742}  &{61.0034} &{60.3424}	   &{65.0307}\\
&\uppercase\expandafter{\romannumeral2} &{65.2465} &{61.6973}	&{60.5743}  &\textbf{68.2007}	&{67.229}  &{62.8342} &{62.6399}	   &{67.0787}\\
&\uppercase\expandafter{\romannumeral3} &\textbf{64.1046} &{59.6485}	&{55.7632}	&{61.1161}	&{60.2489}	 &{62.4145} &{59.6716}	   &{63.8864}\\
&\uppercase\expandafter{\romannumeral4} &\textbf{64.9737} &{60.7835}	&{57.1092}	&{63.2211}	&{64.4834}	 &{61.0204} &{60.364}	   &{64.9525}\\
\hline
\end{tabular}
\end{table*}
\begin{table*}[th]
\centering
\caption{The SSIM comparison values for Fig. \ref{TU1}.
 Types of noise: I- Gaussian noise, II- Poisson noise, III- Salt \& Pepper noise and IV- Speckle noise.}
\label{tabS-1}
\begin{tabular}{cccccccccc}
\hline
&  &QDPC \cite{B26} &QDPA \cite{B26} &Canny \cite{Canny1986} &Sobel \cite{Sobel} &Prewitt \cite{Prewitt}  &DPC \cite{DPC}  &MDPC \cite{MDPC}   &Ours\\
\hline
&\uppercase\expandafter{\romannumeral1} &{0.7058} &{0.5864}	&{0.3912}	&{0.7671}	&{0.7695}   &{0.4018} &{0.4264}	   &\textbf{0.8272}\\
&\uppercase\expandafter{\romannumeral2} &{0.7185} &{0.6347}	&{0.6417}   &{0.8608}	&{0.8649}   &{0.6327} &{0.6143}	   &\textbf{0.9066}\\
&\uppercase\expandafter{\romannumeral3} &{0.6196} &{0.5122}	&{0.2738}	&{0.4207}	&{0.5736}	&{0.2907} &{0.2704}	   &\textbf{0.7275}\\
&\uppercase\expandafter{\romannumeral4} &{0.6866} &{0.5737}	&{0.4438}	&{0.6171}	&{0.7571}	&{0.4538} &{0.4577}	   &\textbf{0.7939}\\
\hline
\end{tabular}
\end{table*}

For any fixed $s_1, s_2 \geq 0$, denote by $f_H(x_1,x_2,s_1,s_2)$ the output signal of the QHF with input $f(x_1,x_2)$. By the definition, we have
\begin{equation}\label{QFT f_H}
\begin{aligned}
&(\mathcal{F}{f_H})(w_1,w_2,s_1,s_2)\\
= &[1+\mbox{sgn}(w_1)][1+\mbox{sgn}(w_2)]
e^{-\mid{w_1}\mid{s_1}}e^{-\mid{t_2}\mid{s_2}}\\
&(\mathcal{F}{f})(w_1,w_2).
\end{aligned}
\end{equation}
Here, the QFT acts on the variable $x_1,x_2$. We will show that as a function of $z_1=x_1+\mathbf{i}s_1$ and $ z_2=x_2+\mathbf{j}s_2$, $f_H$ belongs to the quaternion Hardy space $\mathbf{Q}^2(\mathbb{C}_{\mathbf{i} \mathbf{j}}^+)$.\\

\begin{theorem}
{Let $f\in L^2(\mathbb{R}^2, \mathbb{H})$ and  $f_H$ be given above. Then   $f_H\in\mathbf{Q}^2(\mathbb{C}_{\mathbf{i} \mathbf{j}}^+)$.}
\end{theorem}

\begin{proof}
Using inverse quaternion Fourier transform defined by Eq. (\ref{Eq.inverse 2-sided QFT}), we have that
\begin{equation}\label{proof}
\begin{aligned}
&f_H(x_1,x_2,s_1,s_2)\\
=&\frac{1}{2\pi}\int_{\mathbb{R}^{2}}e^{\mathbf{i}w_1x_1}[1+\mbox{sgn}(w_1)]
[1+\mbox{sgn}(w_2)]e^{-\mid{w_1}\mid{s_1}}\\
&e^{-\mid{w_2}\mid{s_2}}
(\mathcal{F}{f})(w_1,w_2)e^{\mathbf{j}w_2x_2}d{w_1}d{w_2}.
\end{aligned}
\end{equation}
Taking the derivative of $f_H$ with respect to $\overline{z_{1}}$, we get
\begin{equation}\label{proof1}
\begin{aligned}
&\frac{\partial}{\partial \overline{z_{1}}}f_H(x_1,x_2,s_1,s_2)\\
=&[\frac{\partial}{\partial {x_{1}}}+\mathbf{i}\frac{\partial}{\partial {s_{1}}}]f_H(x_1,x_2,s_1,s_2)\\
=&\frac{1}{2\pi}\int_{\mathbb{R}^{2}}\mathbf{i}(w_1-|w_1|)
e^{\mathbf{i}w_1x_1}e^{-\mid{w_1}\mid{s_1}}e^{-\mid{w_2}\mid{s_2}}\\
&[1+\mbox{sgn}(w_1)][1+\mbox{sgn}(w_2)](\mathcal{F}{f})(w_1,w_2)\\
&e^{\mathbf{j}w_2x_2}d{w_1}d{w_2}\\
=&0.
\end{aligned}
\end{equation}
The last equality holds since the integrand vanishes identically. Similarly,
\begin{equation}\label{proof2}
\begin{aligned}
&f_H(x_1,x_2,s_1,s_2)\frac{\partial}{\partial \overline{z_{2}}}\\
=&\frac{\partial}{\partial {x_{2}}}f_H(x_1,x_2,s_1,s_2)+\frac{\partial}{\partial {s_{2}}}f_H(x_1,x_2,s_1,s_2)\mathbf{j}\\
=&\frac{1}{2\pi}\int_{\mathbb{R}^{2}}e^{\mathbf{i}w_1x_1}[1+\mbox{sgn}(w_1)][1+\mbox{sgn}(w_2)]\\
&e^{-\mid{w_1}\mid{s_1}}e^{-\mid{w_2}\mid{s_2}}(\mathcal{F}{f})(w_1,w_2)\mathbf{j}(w_2-|w_2|)\\
&e^{\mathbf{j}w_2x_2}d{w_1}d{w_2}\\
=&0.
\end{aligned}
\end{equation}
For any fixed $s_1>0, s_2>0$, from (\ref{QFT f_H}) we can obtain that
\begin{equation}\label{proof31}
(\mathcal{F}{f_H})(w_1,w_2,s_1,s_2)\leq 4(\mathcal{F}{f})(w_1,w_2).
\end{equation}
According to the QFT Parseval's identity \cite{B35}, we obtain that
\begin{equation}\label{proof32}
\begin{aligned}
&\int_{\mathbb{R}^{2}}|(\mathcal{F}{f_H})(w_1,w_2,s_1,s_2)|^2dw_1dw_2\\
=&\int_{\mathbb{R}^{2}} |f_H(x_1,x_2,s_1,s_2)|^2dx_1dx_2,
\end{aligned}
\end{equation}
\begin{equation}\label{proof33}
\begin{aligned}
&\int_{\mathbb{R}^{2}}|(\mathcal{F}{f})(w_1,w_2)|^2dw_1dw_2\\
=&\int_{\mathbb{R}^{2}} |f(x_1,x_2)|^2dx_1dx_2.
\end{aligned}
\end{equation}
\begin{table*}[th]
\centering
\caption{{The PSNR comparison values for Fig. \ref{TU2}.
 Types of noise: 0-no noise, I- Gaussian noise, II- Poisson noise, III- Salt \& Pepper noise and IV- Speckle noise.}}
\label{tabP-2}
\begin{tabular}{cccccccccc}
\hline
&  &QDPC \cite{B26} &QDPA \cite{B26}   &Canny \cite{Canny1986} &Sobel \cite{Sobel} &Prewitt \cite{Prewitt}  &DPC \cite{DPC}  &MDPC \cite{MDPC}   &Ours\\
\hline
&\uppercase\expandafter{\romannumeral1} &{23.0668} &{23.0547}	&{23.0997}	&{23.0433}	&{23.0432} &{23.0396} &{23.0396}	   &\textbf{23.1006}\\
&\uppercase\expandafter{\romannumeral2} &{23.0776} &{23.0612}	&{23.0791} &{23.0447}	&{23.0447} &{23.0399} &{23.0398}	   &\textbf{23.1131}\\
&\uppercase\expandafter{\romannumeral3} &{23.0719} &{23.0585}	&{23.1006}	&{23.0462}	&{23.0459}	&{23.0400} &{23.0399}	   &\textbf{23.1009}\\
&\uppercase\expandafter{\romannumeral4} &{23.0803} &{23.0574}	&{23.0997}	&{23.0426}	&{23.0427}	&{23.0396} &{23.0396}	   &\textbf{23.1018}\\
\hline
\end{tabular}
\end{table*}
\begin{table*}[th]
\centering
\caption{{The SSIM comparison values for Fig. \ref{TU2}.
 Types of noise: 0-no noise, I- Gaussian noise, II- Poisson noise, III- Salt \& Pepper noise and IV- Speckle noise.}}
\label{tabS-2}
\begin{tabular}{cccccccccc}
\hline
&  &QDPC \cite{B26} &QDPA \cite{B26}   &Canny \cite{Canny1986} &Sobel \cite{Sobel} &Prewitt \cite{Prewitt}  &DPC \cite{DPC}  &MDPC \cite{MDPC}   &Ours\\
\hline
&\uppercase\expandafter{\romannumeral1} &{0.2729} &{0.2179}	&{0.0011}	&{0.1575}	&{0.1577} &{0.2745} &{0.2786}	   &\textbf{0.3061}\\
&\uppercase\expandafter{\romannumeral2} &{0.2992} &{0.2296}	&{0.0205}   &{0.2133}	&{0.2155} &{0.2929} &{0.2931}	   &\textbf{0.3177}\\
&\uppercase\expandafter{\romannumeral3} &{0.1659} &{0.1409}	&{0.0013}	&{0.1025}	&{0.1113}	&{0.2600} &{0.2618}	   &\textbf{0.2941}\\
&\uppercase\expandafter{\romannumeral4} &{0.1788} &{0.1750}	&{0.0008}	&{0.1482}  	&{0.1485}	&{0.2746} &{0.2739}	   &\textbf{0.3084}\\
\hline
\end{tabular}
\end{table*}
\begin{table*}[th]
\centering
\caption{{The PSNR comparison values for Fig. \ref{TU3}.
 Types of noise: 0-no noise, I- Gaussian noise, II- Poisson noise, III- Salt \& Pepper noise and IV- Speckle noise.}}
\label{tabP-3}
\begin{tabular}{cccccccccc}
\hline
&  &QDPC \cite{B26} &QDPA \cite{B26}   &Canny \cite{Canny1986} &Sobel \cite{Sobel} &Prewitt \cite{Prewitt}  &DPC \cite{DPC}  &MDPC \cite{MDPC}   &Ours\\
\hline
&\uppercase\expandafter{\romannumeral1} &{56.9960} &{57.8936}	&{54.0672}	&{60.4140}	&{60.5610}  &{54.1521} &{54.1476}	   &\textbf{62.4188}\\
&\uppercase\expandafter{\romannumeral2} &{58.6229} &{60.0162}	&{54.9847}  &{60.2876}	&{60.3119}  &{56.4325} &{56.4360}	   &\textbf{64.0926}\\
&\uppercase\expandafter{\romannumeral3} &{56.2743} &{57.2830}	&{53.7631}	&{59.3229}	&{59.6852}	&{54.6971} &{54.6898}	   &\textbf{61.6564}\\
&\uppercase\expandafter{\romannumeral4} &{56.4520} &{57.5844}	&{53.9876}	&{60.0839}	&{60.2857}	&{54.3692} &{54.4606}	   &\textbf{61.6499}\\
\hline
\end{tabular}
\end{table*}

\begin{table*}[th]
\centering
\caption{{The SSIM comparison values for Fig. \ref{TU3}.
 Types of noise: 0-no noise, I- Gaussian noise, II- Poisson noise, III- Salt \& Pepper noise and IV- Speckle noise.}}
\label{tabS-3}
\begin{tabular}{cccccccccc}
\hline
&  &QDPC \cite{B26} &QDPA \cite{B26}   &Canny \cite{Canny1986} &Sobel \cite{Sobel} &Prewitt \cite{Prewitt}  &DPC \cite{DPC}  &MDPC \cite{MDPC}   &Ours\\
\hline
&\uppercase\expandafter{\romannumeral1} &{0.3840} &{0.4297}	&{0.0469}	&{0.4648}	&{0.5199}   &{0.0669} &{0.0660}	   &\textbf{0.7318}\\
&\uppercase\expandafter{\romannumeral2} &{0.5361} &{0.6141}	&{0.1727}   &{0.5758}	&{0.5909}   &{0.3027} &{0.3022}	   &\textbf{0.7995}\\
&\uppercase\expandafter{\romannumeral3} &{0.2921} &{0.3647}	&{0.0190}	&{0.2277}	&{0.2990}	&{0.0675} &{0.0694}	   &\textbf{0.6612}\\
&\uppercase\expandafter{\romannumeral4} &{03295}  &{0.4007}	&{0.0646}	&{0.4200}   &{0.4643}	&{0.1168} &{0.1206}	   &\textbf{0.6577}\\
\hline
\end{tabular}
\end{table*}
\begin{table*}[th]
\centering
\caption{The PSNR comparison values for Fig. \ref{TU4}.
 Types of noise: I- Gaussian noise, II- Poisson noise, III- Salt \& Pepper noise and IV- Speckle noise.}
\label{tabP-4}
\begin{tabular}{cccccccccc}
\hline
&  &QDPC \cite{B26} &QDPA \cite{B26}   &Canny \cite{Canny1986} &Sobel \cite{Sobel} &Prewitt \cite{Prewitt}  &DPC \cite{DPC}  &MDPC \cite{MDPC}   &Ours\\
\hline
&\uppercase\expandafter{\romannumeral1} &{59.5875} &{57.1603}	&{54.8202}	&{61.4068}	&{61.8794}  &{58.5991} &{59.1803}	   &\textbf{62.5796}\\
&\uppercase\expandafter{\romannumeral2} &{59.7487} &{58.0289}	&{56.0615}  &{61.7013}	&{62.0238}  &{59.1115} &{59.5553}	   &\textbf{64.3704}\\
&\uppercase\expandafter{\romannumeral3} &{58.9160} &{56.4631}	&{54.2282}	&{59.9397}	&{60.7958}	&{57.5608} &{58.4340}	   &\textbf{61.6690}\\
&\uppercase\expandafter{\romannumeral4} &{59.4685} &{57.4581}	&{54.7294}	&{60.3342}	&{61.5622}	&{58.4213} &{59.1584}	   &\textbf{62.9552}\\
\hline
\end{tabular}
\end{table*}

\begin{table*}[th]
\centering
\caption{The SSIM comparison values for Fig. \ref{TU4}.
 Types of noise: I- Gaussian noise, II- Poisson noise, III- Salt \& Pepper noise and IV- Speckle noise.}
\label{tabS-4}
\begin{tabular}{cccccccccc}
\hline
&  &QDPC \cite{B26} &QDPA \cite{B26} &Canny \cite{Canny1986} &Sobel \cite{Sobel} &Prewitt \cite{Prewitt}  &DPC \cite{DPC}  &MDPC \cite{MDPC}   &Ours\\
\hline
&\uppercase\expandafter{\romannumeral1} &{0.5250} &{0.3732}	&{0.1650}	&{0.5817}	&{0.6089}   &{0.4226} &{0.4711}	   &\textbf{0.7276}\\
&\uppercase\expandafter{\romannumeral2} &{0.5292} &{0.4451}	&{0.3232}   &{0.6508}	&{0.6642}   &{0.4768} &{0.5207}	   &\textbf{0.8188}\\
&\uppercase\expandafter{\romannumeral3} &{0.3871} &{0.2828}	&{0.0920}	&{0.3426}	&{0.3897}	&{0.2471} &{0.3543}	   &\textbf{0.6602}\\
&\uppercase\expandafter{\romannumeral4} &{0.5160} &{0.3899}	&{0.1673}	&{0.4769}	&{0.5877}	&{0.4103} &{0.4779}	   &\textbf{0.7536}\\
\hline
\end{tabular}
\end{table*}

 Using (\ref{proof31}), (\ref{proof32}) and (\ref{proof33}), a direct computation shows that
\begin{equation}\label{proof34}
\begin{aligned}
 &\sup\limits_{\substack{s_1>0 \\
 s_2>0}}\int_{\mathbb{R}^2} |f_H(x_1+ s_1\mathbf{i}, x_2+s_2\mathbf{j} )|^{2}dx_{1}dx_{2} \\
 \leq&   \sup\limits_{\substack{s_1>0 \\  s_2>0}} 16\int_{\mathbb{R}^{2}}|(\mathcal{F}{f})(w_1,w_2)|^2dw_1dw_2
 < \infty.
\end{aligned}
\end{equation}
The proof is complete.
\end{proof}

\subsection{Color edge detection algorithm}
\label{Pro-CEDA}
In this section, the edge detector based on QHF  are described.
{First, we obtain the high-dimensional analytic signal of the original image through QHF. Second, the obtained high-dimensional analytic signal is used as the input of IDZ gradient operator. Finally, we get the result of edge detection.}
Let us now give the details of the quaternion Hardy filter based algorithm. They are divided by the following steps.

\begin{itemize}
  \item []
  {\bf{Step} 1}. {Given an input digital color image $f$ of size $M\times N$. Associate it with a $\mathbb{H}$-valued signal}
  \begin{eqnarray}\label{f}
  f=f_1\mathbf{i}+f_2\mathbf{j}+f_3\mathbf{k},
  \end{eqnarray}
  where $f_1,f_2$ and $f_3$ represent respectively the red, green and blue components of the given color image.
  \item []
  {\bf{Step} 2}. Compute the DQFT of the $f$ using equation (\ref{DQFT}). The result will be $\mathcal{F}_D[f]$.
  \item []
  {\bf{Step} 3}. For fixed $s_1>0, s_2>0$ (the values of parameters $s_1$ and $s_2$ of the original image ranged from 1.0 to 2.0, and those of the noisy image ranged from 1.0 to 8.0, multiplying $\mathcal{F}_D[f]$ by the system function (\ref{filter2}) of the QHF. Then we obtain the DQFT of $f_H$ which has the following form
  \begin{equation}
  \begin{aligned}
  &(\mathcal{F}_D[f_H])(w_1,w_2,s_1,s_2)\\
  =& [1+\mbox{sgn}(w_1)]
  [1+\mbox{sgn}(w_2)]e^{-\mid{w_1}\mid{s_1}}e^{-\mid{t_2}\mid{s_2}}\\
  &(\mathcal{F}_D{f})(w_1,w_2).
  \end{aligned}
  \end{equation}

  This is the most significant step in our algorithm, because it allows the advantages of QHF to be presented.

  \item []
  {\bf{Step} 4}. Compute the inverse DQFT for  $\mathcal{F}_D[f_H]$ by applying equation (\ref{IDQFT}), we obtain $f_H$.
  \item []
  {\bf{Step} 5}.  Extract the vector part of $f_H$, we obtain
  \begin{eqnarray}
  \mathbf{Vec}(f_H)=h_1\mathbf{i}+h_2\mathbf{j}+h_3\mathbf{k},
  \end{eqnarray}
  where $h_k$, $k=1,2,3$ are real-valued functions.
 {In the following, we will operate IDZ algorithm based on $\mathbf{Vec}(f_H)$ instead of $f$.}
  \item []
  {\bf{Step} 6}. Perform the IDZ gradient operator to $\mathbf{Vec}(f_H)$. Applying equation (\ref{ABC}), we obtain
  \begin{eqnarray}\label{ABC-h}
  \left\{
  \begin{aligned}
  &A= \sum\limits_{k=1}^{3}(\frac{\partial{h_k}}{\partial{x_1}})^2; \\
  &B= \sum\limits_{k=1}^{3}(\frac{\partial{h_k}}{\partial{x_2}})^2; \\
  &C= \sum\limits_{k=1}^{3}\frac{\partial{h_k}}{\partial{x_1}}\frac{\partial{h_k}}{\partial{x_2}},
  \end{aligned}\right.
  \end{eqnarray}

  then we substitute them into equation (\ref{fMAX}), obtain
  \begin{equation}\label{fMAX1}
  \begin{aligned}
    &\mathbf{Vec}(f_H)_{\max}\\
    =&\frac{1}{2}\bigg(A+C+\sqrt{(A-C)^2+(2B)^2}\bigg).
  \end{aligned}
  \end{equation}

  \item []
  {\bf{Step} 7}. Finally, we obtain the processed result, edge map, by applying the nonmaxmum suppress.

\end{itemize}
{We assume, for simplicity, that $M= N$. It is easy to see that the main computation complexity lies in QFT. According to  \cite{bib19}, the computational complexity of QFT for a $M\times N$ image is about $\mathcal{O}(M^2\log_{2}M)$.
The computational complexity of the remaining steps is approximately $\mathcal{O}(M^3)$. Therefore, the whole computational complexity of the proposed algorithm is about $\mathcal{O}(M^2\log_{2}M)+\mathcal{O}(M^3)$. While the computational cost of MDPC algorithm is about $\mathcal{O}(M^2\log_{2}M)$. The computational complexity of both QDPC and QDPA algorithm is $\mathcal{O}(M^2\log_{2}M)$.}

\section{Experimental results}
\label{sec:Exp}
In this section, we shall demonstrate the effectiveness of the proposed algorithm for color image edge detection.

\begin{table}[th]
\centering
\caption{The experimental setups of the proposed algorithm.}
\label{algorithm}
\begin{tabular}{ccc}
\hline
&Name &Quaternion Hardy filter edge detection algorithm \\
\hline
&Tool: &Matlab R2016b\\
&Dataset 1: &Fig. \ref{TU1}\\
&Dataset 2: &Fig. \ref{TU2}\\
&Dataset 3: &Fig. \ref{TU3}\\
&Dataset 4: &Fig. \ref{TU4}\\
\hline
\end{tabular}
\end{table}

{Here both visual and quantitative analysis for edge detection are considered in our experiments. All experiments are programmed in Matlab R2016b.
To validate the effectiveness of the proposed method, we have carried out verification on 4 datasets (Fig. \ref{TU1}, \ref{TU2}, \ref{TU3}, and \ref{TU4}).
The image in Fig. \ref{TU1} is what we call Classic in this article.
The images in Fig.\ref{TU2} are randomly selected from the public image dataset (namely Synthetic, and from \href{https://urldefense.proofpoint.com/v2/url?u=http-3A__color.univ-2Dlille.fr_datasets_color-2Dedge&d=DwIFaQ&c=KXXihdR8fRNGFkKiMQzstu-8MbOxd1NuZkcSBymGmgo&r=WuCkcl8HOAexOSwshlwJ3w&m=GpWzZCbpvddgVNECF4m_mlofU3pS9zmK0WK-eRKLN2Q&s=_GRVRNnmdJv7DBtB1vHs6n-mw2OpOT-FiCx5e3fl0p4&e=}
{link 2}).
The test images in Fig. \ref{TU3} are randomly selected from the public image dataset (namely CSEE, and from \cite{bibre2}).
The images in Fig. \ref{TU4} are randomly selected from the public image dataset (namely Dataset 4, and from \href{http://decsai.ugr.es/cvg/dbimagenes/}
{link 4}), which has been used by previous researchers.  It consists of 805 test images with 3 different size scales.
Here, the Gaussian filter \cite{gaussian1, gaussian2} is applied to these algorithms (Canny, Sobel, Prewitt, DPC, and MDPC), since they doesn't have the ability of resisting noise.
Digital images distorted with different types of noise such as I- Gaussian noise \cite{gaussian}, II- Poisson noise, III- Salt \& Pepper noise, and IV- Speckle noise. The ideal noiseless and noisy images are both taken into account.}
\subsection{Visual comparisons}
In terms of visual analysis, a color-based method IDZ and seven widely used and noteworthy methods QDPC, QDPA, Canny, Sobel, Prewitt, Differential Phase Congruence (DPC) and Modified Differential Phase Congruence (MDPC)  will be compared with our algorithm.
\subsubsection{Color-based algorithm}
\label{Exp-color}
In this part, we compare the proposed algorithm with the IDZ gradient algorithm. In order to make the experiment more convincing, we used Gaussian filter before IDZ algorithm to achieve the effect of denoising.
Fig. \ref{house1} presents the edge map of the noiseless House image, while Fig. \ref{house2} presents the edge map of the House image corrupted with four different types of noise.
It can be seen from the second row of Fig. \ref{house2} that IDZ gradient algorithm performs well in the first two images of the first line, while poorly in the last two images.
This illustrates that the IDZ gradient algorithm's limitations as a edge detector. The third row of Fig. \ref{house2} shows the detection result of the proposed algorithm. It preserves details more clearly than the second row. It demonstrates that the proposed algorithm gives robust performance compared to that of the IDZ gradient algorithm.

\subsubsection{Grayscale-based algorithms}
\label{Exp-Grayscale}
We compare the performance of the proposed algorithm  with seven widely used and noteworthy algorithms. The  noiseless (Fig. \ref{TU1}, \ref{TU2}, \ref{TU3}, and \ref{TU4}) and noisy images  are both taken into consideration. Here, the commonly used color-to-gray conversion formula \cite{gray1, gray2} is applied in the experiments, which is defined as follows
\begin{equation}
  Gray=0.299*R+0.587*G+0.114*B.
\end{equation}

\begin{itemize}
\item {\bf Noiseless case: }
  {In total, we selected two images from four data sets as test images to show the edge detection effect under noiseless conditions.
  In other words, an image in the first row and second column is selected from Fig. \ref{TU2}, and then an image in the first column of the third row is selected from Fig. \ref{TU3}.
  Fig. \ref{NL0} and Fig. \ref{N30} respectively show the edge detection results of these two image.
  Different rows and columns correspond to the results of different methods. From left to right and top to bottom they are QDPC, QDPA, Canny, Sobel, Prewitt, DPC, MDPC and the proposed algorithms, respectively. The results show that these eight methods are ideal for noiseless images.}
\item {\bf Noisy case: }
  {Here, we add four different noises (I-IV) to each image selected above.
  The edge maps obtained by applying the QDPC, QDPA, Canny, Sobel, Prewitt, DPC, MDPC and the proposed methods to noisy images are shown in Fig. \ref{NL1}, Fig. \ref{NL2}, Fig. \ref{NL3}, Fig. \ref{NL4}, Fig. \ref{N31}, Fig. \ref{N32}, Fig. \ref{N33}, and Fig. \ref{N34}, respectively. we can clearly see that the proposed algorithm is able to extract edge maps from  the noisy images. This means that the proposed algorithm is resistant to the noise.  In particular, it is superior to the other detectors on images with noise III and IV.}
\label{noise}
\end{itemize}

\begin{table}
  \caption{Parameter settings for image Lena in Table \ref{tabP-1} and Table \ref{tabS-1}.}
  \label{pa-1}
  \centering
   \begin{tabular}{c| c c c c c}
\hline
Noise &QDPC \cite{B26}  &QDPA \cite{B26}  &DPC \cite{DPC}   & MDPC \cite{MDPC}  & Ours                 \\
\hline
\uppercase\expandafter{0}                &$(1.0,1.0)$ &$(1.0,1.0)$ &$0.5$    &$0.5$    &$(2.0,2.0)$\\
\uppercase\expandafter{\romannumeral1}   &$(4.5,4.5)$ &$(4.0,4.0)$ &$4.0$  &$3.5$ & $(7.0,7.0)$   \\
\uppercase\expandafter{\romannumeral2}   &$(3.5,3.5)$ &$(3.0,3.0)$ &$2.0$  &$2.5$ & $(6.0,6.0)$    \\
\uppercase\expandafter{\romannumeral3}   &$(4.5,4.5)$ &$(4.5,4.5)$ &$5.0$  &$4.5$ & $(7.0,7.0)$    \\
\uppercase\expandafter{\romannumeral4}   &$(4.5,4.5)$ &$(4.5,4.5)$ &$5.0$  &$4.5$ & $(7.0,7.0)$   \\
\hline
\end{tabular}
\end{table}

\begin{table}
  \caption{Parameter settings for image Men in Table \ref{tabP-1} and Table \ref{tabS-1}.}
  \label{pa-2}
  \centering
  \begin{tabular}{c|ccccc}
\hline
Noise &QDPC \cite{B26}  &QDPA \cite{B26}  &DPC \cite{DPC}   & MDPC \cite{MDPC}  & Ours  \\
\hline
\uppercase\expandafter{0}              &$(2.0,2.0)$ &(2.0,2.0) &1.0   &$0.5$    &$(2.0,2.0)$  \\
\uppercase\expandafter{\romannumeral1} &$(6.5,6.5)$ &(6.5,6.5) &5.5   &$3.5$    &$(5.5,5.5)$  \\
\uppercase\expandafter{\romannumeral2} &$(5.5,5.5)$ &(5.5,5.5) &3.0   &$2.5$    &$(5.5,5.5)$   \\
\uppercase\expandafter{\romannumeral3} &$(7.0,7.0)$ &(7.5,7.5) &6.0   &$4.5$    &$(5.5,5.5)$    \\
\uppercase\expandafter{\romannumeral4} &$(6.5,6.5)$ &(6.5,6.5) &5.5   &$4.5$    &$(5.5,5.5)$   \\
\hline
\end{tabular}
\end{table}

\begin{table}
\caption{Parameter settings for image House in Table \ref{tabP-1} and Table \ref{tabS-1}.}
\label{pa-3}
  \centering
  \begin{tabular}{c|ccccc}
\hline
Noise &QDPC \cite{B26}  &QDPA  \cite{B26} &DPC \cite{DPC}   & MDPC \cite{MDPC}  & Ours  \\
\hline
\uppercase\expandafter{0}              &$(2.0,2.0)$  &(2.0,2.0)  &0.5  &$0.5$  &$(2.0,2.0)$   \\
\uppercase\expandafter{\romannumeral1} &$(7.0,7.0)$  &(7.0,7.0)  &4.0  &$3.5$  &$(8.0,8.0)$   \\
\uppercase\expandafter{\romannumeral2} &$(6.5,6.5)$  &(6.5,6.5)  &2.5  &$2.5$  &$(6.0,6.0)$   \\
\uppercase\expandafter{\romannumeral3} &$(8.0,8.0)$  &(8.0,8.0)  &5.0  &$4.5$  &$(8.0,8.0)$    \\
\uppercase\expandafter{\romannumeral4} &$(7.0,7.0)$  &(7.5,7.5)  &5.0  &$4.5$  &$(8.0,8.0)$    \\
\hline
\end{tabular}
\end{table}

\begin{table}
\caption{Parameter settings for image T1 in Table \ref{tabP-1} and Table \ref{tabS-1}.}
\label{pa-4}
  \centering
  \begin{tabular}{c|ccccc}
\hline
Noise &QDPC \cite{B26} &QDPA \cite{B26} &DPC \cite{DPC}   & MDPC \cite{MDPC}  & Ours  \\
\hline
\uppercase\expandafter{0}              &$(2.0,2.0)$  &(2.0,2.0)  &0.5  &$0.5$   &$(2.0,2.0)$   \\
\uppercase\expandafter{\romannumeral1} &$(5.0,5.0)$  &(5.0,5.0)  &3.0  &$2.0$   &$(5.0,5.0)$  \\
\uppercase\expandafter{\romannumeral2} &$(4.5,4.5)$  &(4.5,4.5)  &2.0  &$2.0$   &$(5.0,5.0)$    \\
\uppercase\expandafter{\romannumeral3} &$(6.0,6.0)$  &(6.5,6.5)  &3.5  &$2.0$   &$(5.0,5.0)$    \\
\uppercase\expandafter{\romannumeral4} &$(5.5,5.5)$  &(6.0,6.0)  &3.0  &$2.0$   &$(5.0,5.0)$    \\
\hline
\end{tabular}
\end{table}

\begin{table}
\caption{ Parameter settings for image T2 in Table \ref{tabP-1} and Table \ref{tabS-1}.}
\label{pa-5}
  \centering
  \begin{tabular}{c|ccccc}
\hline
Noise &QDPC \cite{B26} &QDPA \cite{B26}  &DPC \cite{DPC}   & MDPC \cite{MDPC}  & Ours  \\
\hline
\uppercase\expandafter{0}               &$(2.5,2.5)$  &(2.5,2.5)  &0.5  &$0.5$  &$(2.0,2.0)$   \\
\uppercase\expandafter{\romannumeral1}  &$(3.5,3.5)$  &(3.5,3.5)  &0.5  &$0.5$  &$(2.0,2.0)$   \\
\uppercase\expandafter{\romannumeral2}  &$(2.5,2.5)$  &(2.0,2.0)  &0.5  &$0.5$  &$(2.0,2.0)$   \\
\uppercase\expandafter{\romannumeral3}  &$(4.5,4.5)$  &(4.5,4.5)  &0.5  &$0.5$  &$(2.0,2.0)$    \\
\uppercase\expandafter{\romannumeral4}  &$(3.5,3.5)$  &(4.0,4.0)  &0.5  &$0.5$  &$(2.0,2.0)$    \\
\hline
\end{tabular}
\end{table}

\begin{table}
 \caption{Parameter settings for image T3 in Table \ref{tabP-1} and Table \ref{tabS-1}.}
 \label{pa-6}
  \centering
\begin{tabular}{c|ccccc}
\hline
Noise &QDPC \cite{B26}  &QDPA \cite{B26} &DPC \cite{DPC}   & MDPC \cite{MDPC}  & Ours   \\
\hline
\uppercase\expandafter{0}              &$(2.5,2.5)$  &(2.5,2.5) &0.5  &$0.5$  &$(2.0,2.0)$   \\
\uppercase\expandafter{\romannumeral1} &$(6.5,6.5)$  &(8.5,8.5) &8.5  &$7.0$  &$(8.0,8.0)$  \\
\uppercase\expandafter{\romannumeral2} &$(5.5,5.5)$  &(5.5,5.5) &4.5  &$5.5$  &$(6.0,6.0)$   \\
\uppercase\expandafter{\romannumeral3} &$(7.5,7.5)$  &(7.5,7.5) &8.5  &$8.0$  &$(8.0,8.0)$    \\
\uppercase\expandafter{\romannumeral4} &$(7.5,7.5)$  &(7.5,7.5) &8.5  &$8.0$  &$(8.0,8.0)$    \\
\hline
\end{tabular}
\end{table}

\begin{table}
 \caption{Parameter settings for image Cara in Table \ref{tabP-1} and Table \ref{tabS-1}.}
 \label{pa-7}
  \centering
\begin{tabular}{c|ccccc}
\hline
Noise &QDPC \cite{B26}  &QDPA \cite{B26}  &DPC \cite{DPC}   & MDPC \cite{MDPC}  & Ours   \\
\hline
\uppercase\expandafter{0}              &$(2.5,2.5)$  &(2.5,2.5) &0.5  &$0.5$  &$(2.5,2.5)$   \\
\uppercase\expandafter{\romannumeral1} &$(7.5,7.5)$  &(7.5,7.5) &7.5  &$7.0$  &$(8.0,8.0)$  \\
\uppercase\expandafter{\romannumeral2} &$(3.5,3.5)$  &(4.5,4.5) &5.5  &$5.5$  &$(6.0,6.0)$   \\
\uppercase\expandafter{\romannumeral3} &$(8.5,8.5)$  &(8.5,8.5) &8.5  &$8.0$  &$(8.0,8.0)$    \\
\uppercase\expandafter{\romannumeral4} &$(7.5,7.5)$  &(7.0,7.0) &8.5  &$8.0$  &$(8.0,8.0)$    \\
\hline
\end{tabular}
\end{table}

\begin{table}
 \caption{Parameter settings for image Frog in Table \ref{tabP-1} and Table \ref{tabS-1}.}
 \label{pa-8}
  \centering
\begin{tabular}{c|ccccc}
\hline
Noise &QDPC \cite{B26} &QDPA \cite{B26}  &DPC \cite{DPC}    & MDPC \cite{MDPC}  & Ours   \\
\hline
\uppercase\expandafter{0}              &$(2.5,2.5)$  &(2.5,2.5) &0.5  &$0.5$  &$(2.0,2.0)$   \\
\uppercase\expandafter{\romannumeral1} &$(6.5,6.5)$  &(5.5,5.5) &7.5  &$7.0$  &$(8.0,8.0)$  \\
\uppercase\expandafter{\romannumeral2} &$(3.5,3.5)$  &(2.5,2.5) &4.5  &$5.5$  &$(5.0,5.0)$   \\
\uppercase\expandafter{\romannumeral3} &$(7.5,7.5)$  &(7.5,7.5) &8.5  &$8.0$  &$(8.0,8.0)$    \\
\uppercase\expandafter{\romannumeral4} &$(8.0,8.0)$  &(8.0,8.0) &8.5  &$8.0$  &$(8.0,8.0)$    \\
\hline
\end{tabular}
\end{table}

\subsection{Quantitative analysis}
\label{Exp-Quant}

{The {PSNR} \cite{psnr} is a widely used method of objective evaluation of two images. It is based on the error-sensitive image quality evaluation.
In addition, the SSIM \cite{ssim} is a method of comparing two images under the three aspects of brightness, contrast and structure.
Table \ref{pa-1} - \ref{pa-8} gives the parameter settings in the comparison experiment of Dataset 1.}

{To show the accuracy of the proposed edge detector, the PSNR  and SSIM average values of various type of edge detectors on noisy images (I- Gaussian noise, II- Poisson noise, III- Salt \& Pepper noise and IV- Speckle noise) are calculated (Table \ref{tabP-1} - \ref{tabS-4}).
Each value in the table represents the similarity between the edge map of the noisy image and the edge map of the noiseless image. That is, the larger the value, the stronger the denoising ability.}

{Table \ref{tabP-1} shows the average PSNR value of experimental results of Fig. \ref{TU1}. In the case of the noise I, Sobel and QDPC have the highest PSNR values, and the method proposed in this paper ranks the third.
In the case of the noise II, Sobel and Prewitt had the highest PSNR, and the method proposed in this paper still ranked the third.
In the case of noises III and IV, QDPC had the highest PSNR value, and the method proposed in this paper was followed by QDPC in second place.}

{In general, the method proposed in this paper ranks the top three among the eight comparison methods. This shows that these three algorithms can achieve high similarity between the edge map of noisy image and noiseless image. Therefore, from the point of view of PSNR value, these three algorithms have excellent robustness than the others.}

{Table \ref{tabS-1} shows the average SSIM value of experimental results of Fig. \ref{TU1}. From the table, it is clear that the method presented in this article has the highest SSIM value. In addition to the proposed method, the SSIM values of QDPC and Prewitt are also very high.}

{Compared with QDPC, the improvement rate of SSIM under four noise conditions was $17.2\%$, $26.6\%$, $17.4\%$, and $15.6\%$. Compared with Prewitt, the SSIM improvement rates of the four noise conditions were $7.5\%$, $4.8\%$, $26.8\%$, and $4.9\%$,, respectively.
On the whole, using the proposed method to do color edge detection on this type of image, the performance is obviously excellent.}

{Table \ref{tabP-2} shows the average PSNR value of the experimental results of Fig. \ref{TU2}.
There is no doubt that the proposed method has the highest value under all noise conditions.
But it's worth noting that all the differences are small.
After calculation, we found that the improvement rate was less than $1\%$ compared to the second best Canny method. This indicates that, from the perspective of PSNR, the proposed edge detection method has only a slight improvement on Fig. \ref{TU2}.}

{Tables  \ref{tabS-2} shows the SSIM values between the ground truth and the edge maps of the noisy images of Fig. \ref{TU2}. We know that the closer the SSIM value is to 1, the better performance of the algorithm is. From the average SSIM  values in the table, the proposed method was optimal and improved by $9.9\%$, $8.4\%$, $12.3\%$, and $12.6\%$, respectively, compared with the second-place method (MDPC) under four noise conditions.}

{Table \ref{tabP-3} shows the average PSNR value of experimental results of Fig. \ref{TU3}.
The PSNR value of the proposed method is the highest, and is $3.1\%$, $6.3\%$, $3.3\%$, and $2.3\%$ higher than that of the second-place method (Prewitt) under four noise conditions.}

{Table \ref{tabS-3} shows the average SSIM value of experimental results of Fig. \ref{TU3}.
The SSIM value of the proposed method is the highest, and is $40.8\%$, and $41.7\%$, higher than that of the second method (Prewitt) under the noise I and IV conditions. For the noise II and III, the proposed method improved by $30.2\%$, and $81.3\%$, respectively, compared with the second-place method (QDPA).}

{The average PSNR values of the experimental results in Fig. \ref{TU4} are shown in Table \ref{tabP-4}.
The optimal and second-place methods are the proposed method and Prewitt, respectively. Compared with Prewitt, the improvement rates of the proposed method under four noise conditions are $1.1\%$, $3.8\%$, $1.4\%$, and $2.3\%$ respectively.}

{The average SSIM values of the experimental results in Fig. \ref{TU4} are shown in Table \ref{tabS-4}.
The top two methods are the proposed method and Prewitt. Compared with the slight increase of PSNR value, SSIM value was improved by $19.5\%$, $23.3\%$, $69.4\%$, and $28.2\%$ respectively under four noise conditions. This indicates that the proposed method maintains the structure of Fig. \ref{TU4} well in the case of four different types of noise.}
%
%

\section{Conclusion}
\label{sec:Con}

{In this paper, we introduces a novel Quaternion Hardy filter technique for color image processing. With the help of this technique, a color edge detection method is developed.
QHF is not only adopted to feature enhancement but also used to handle different types of noise.
Specifically, the IDZ gradient operator is considered, which are compensated for by QHF. The results of four datasets verify the proposed method.}

{It is noted that the proposed algorithm still has limitation.
The proposed method contains both the advantages of QHF and IDZ, as well as their disadvantages.
In other words, compared with other methods, the computational complexity cost of the proposed method is a little larger.
Using low-rank quaternion tensor techniques \cite{miao} might make our approach a reality in terms of reducing computing costs.
In further work, low computation cost QHF-based edge detection method will be investigated.}

\end{document}